\def\year{2020}\relax
\documentclass[letterpaper]{article} 
\usepackage{aaai20}  
\usepackage{times}  
\usepackage{helvet} 
\usepackage{courier}  
\usepackage[hyphens]{url}  
\usepackage{graphicx} 
\usepackage{graphicx}  
\usepackage{amsmath,amsfonts,amssymb,amsthm}
\usepackage{mathtools}
\usepackage{pifont}
\newcommand{\cmark}{\ding{51}}

\usepackage{graphicx}

\newtheorem{theorem}{Theorem}
\newtheorem{definition}{Definition}
\newtheorem{proposition}{Proposition}

\newtheorem{lemma}{Lemma}

\usepackage[bottom]{footmisc}
\usepackage{color}
\usepackage{enumitem}
\usepackage{multirow}
\usepackage{booktabs}
\usepackage{appendix}
\usepackage{subfig}
\usepackage{soul}

\usepackage[linesnumbered,ruled]{algorithm2e}
\newcommand{\V}{\mathcal{V}}
\newcommand{\E}{\mathcal{E}}

\title{ASAP: Adaptive Structure Aware Pooling for Learning Hierarchical Graph Representations}
\begin{document}

\author{
	Ekagra Ranjan\textsuperscript{1}\thanks{ Research done during internship at Indian Institute of Science, Bangalore. },
	Soumya Sanyal\textsuperscript{2},
	Partha Talukdar\textsuperscript{2}\\
	\textsuperscript{1}{Indian Institute of Technology, Guwahati}\\
	\textsuperscript{2}{Indian Institute of Science, Bangalore}\\\
	ekagra.ranjan@gmail.com, \{soumyasanyal, ppt\}@iisc.ac.in
}

\maketitle

\begin{abstract}
Graph Neural Networks (GNN) have been shown to work effectively for modeling graph structured data to solve tasks such as node classification, link prediction and graph classification. There has been some recent progress in defining the notion of pooling in graphs whereby the model tries to generate a graph level representation by downsampling and summarizing the information present in the nodes. Existing pooling methods either fail to effectively capture the graph substructure or do not easily scale to large graphs. In this work, we propose ASAP (Adaptive Structure Aware Pooling), a sparse and differentiable pooling method that addresses the limitations of previous graph pooling architectures. ASAP utilizes a novel self-attention network along with a modified GNN formulation to capture the importance of each node in a given graph. It also learns a sparse soft cluster assignment for nodes at each layer to effectively pool the subgraphs to form the pooled graph. Through extensive experiments on multiple datasets and theoretical analysis, we motivate our choice of the components used in ASAP. Our experimental results show that combining existing GNN architectures with ASAP leads to state-of-the-art results on multiple graph classification benchmarks. ASAP has an average improvement of 4\%, compared to current sparse hierarchical state-of-the-art method.
\end{abstract}

\begin{figure*}[!ht]
	\includegraphics[width=1\textwidth]{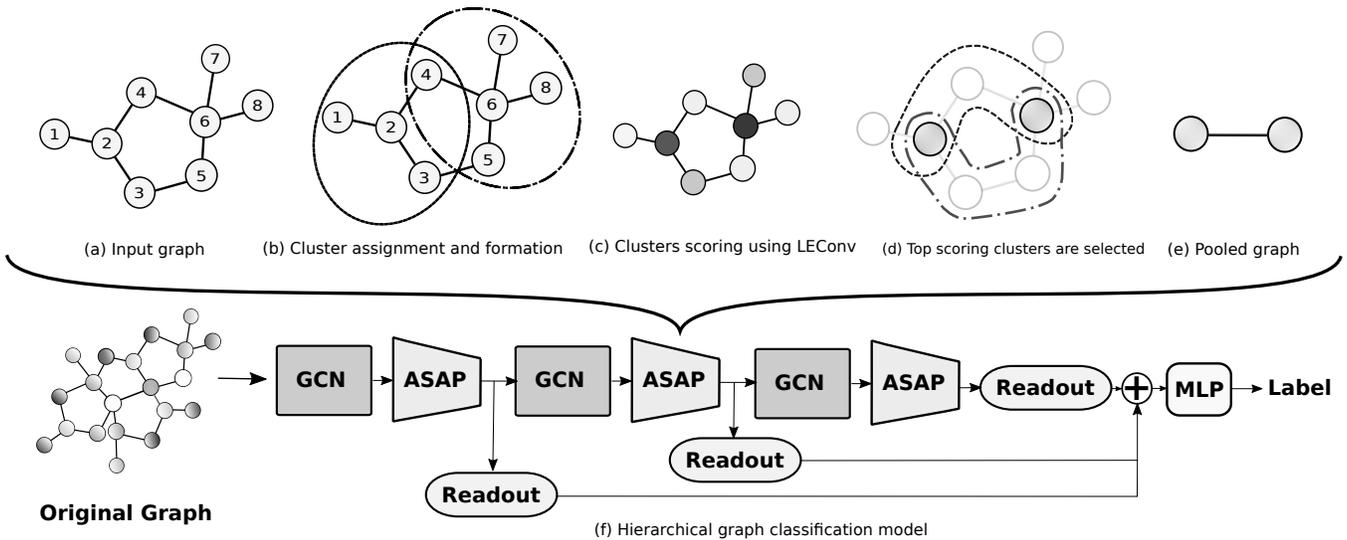}
	\caption{\label{fig:asap} Overview of ASAP: (a) Input graph to ASAP. (b) ASAP initially clusters $1$-hop neighborhood considering all nodes as medoid\protect\footnotemark. For brevity, we only show cluster formations of nodes 2 \& 6 as medoids. Cluster membership is computed using M2T attention (refer Sec. \ref{ssec:node-aggregation}). (c) Clusters are scored using LEConv (refer Sec. \ref{ssec:cluster-selection}). Darker shade denotes higher score. (d) A fraction of top scoring clusters are selected in the pooled graph. Adjacency matrix is recomputed using edge weights between the member nodes of selected clusters. (e) Output of ASAP (f) Overview of hierarchical graph classification architecture.}
\end{figure*}

\section{Introduction}
In recent years, there has been an increasing interest in developing Graph Neural Networks (GNNs) for graph structured data.  CNNs have shown to be successful in tasks involving images \cite{alex,resnet} and text \cite{conv-text}. Unlike these regular grid data, arbitrary shaped graphs have rich information present in their graph structure. By inherently capturing such information through message propagation along the edges of the graph, GNNs have proved to be more effective for graphs \cite{message_passing,sage}. GNNs have been successfully applied in tasks such as semantic role labeling \cite{gcn_srl}, relation extraction \cite{gcn_reside}, neural machine translation \cite{gcn_nmt}, document dating \cite{gcn_document}, and molecular feature extraction \cite{gcn_molecule}. While some of the works focus on learning node-level representations to perform tasks such as node classification \cite{gcn,gat} and link prediction \cite{rgcn,compgcn}, others focus on learning graph-level representations for tasks like graph classification \cite{spectral-gcn-1,spectral-gcn-2,diffpool,topk,sag} and graph regression \cite{cgcnn,mtcgcnn}. In this paper, we focus on graph-level representation learning for the task of graph classification.

Briefly, the task of graph classification involves predicting the label of an input graph by utilizing the given graph structure and initial node-level representations. For example, given a molecule, the task could be to predict if it is toxic. Current GNNs are inherently \emph{flat} and lack the capability of aggregating node information in a \emph{hierarchical} manner. Such architectures rely on learning node representations through some GNN followed by aggregation of the node information to generate the graph representation \cite{set2set,glob-att,sortpool}. But learning graph representations in a hierarchical manner is important to capture local substructures that are present in graphs. For example, in an organic molecule, a set of atoms together can act as a functional group and play a vital role in determining the class of the graph.

To address this limitation, new pooling architectures have been proposed where sets of nodes are recursively aggregated to form a cluster that represents a node in the pooled graph, thus enabling hierarchical learning. DiffPool \cite{diffpool} is a differentiable pooling operator that learns a soft assignment matrix mapping each node to a set of clusters. Since this assignment matrix is \emph{dense}, it is not easily scalable to large graphs \cite{topk2}. Following that, TopK \cite{topk} is proposed which learns a scalar projection score for each node and selects the top $k$ nodes. They address the sparsity concerns of DiffPool but are unable to capture the rich graph structure effectively. Recently, SAGPool \cite{sag}, a TopK based architecture, has been proposed which leverages self-attention network to learn the node scores. Although local graph structure is used for scoring nodes, it is still not used effectively in determining the connectivity of the pooled graph. Pooling methods that leverage the graph structure effectively while maintaining sparsity currently don't exist. We address the gap in this paper.


In this work, we propose a new sparse pooling operator called Adaptive Structure Aware Pooling (ASAP) which overcomes the limitations in current pooling methods. Our contributions can be summarized as follows:
\begin{itemize}
	\item We introduce ASAP, a sparse pooling operator capable of capturing local subgraph information hierarchically to learn global features with better edge connectivity in the pooled graph.
	\item We propose Master2Token (M2T), a new self-attention framework which is better suited for global tasks like pooling.
	\item We introduce a new convolution operator LEConv, that can adaptively learn functions of local extremas in a graph substructure.
\end{itemize}

\footnotetext{medoids are representatives of a cluster. They are similar to centroids but are strictly a member of the cluster.}

\section{Related Work}

\subsection{Graph Neural Networks}
Various formulation of GNNs have been proposed which use both spectral and non-spectral approaches. Spectral methods \cite{spectral-gcn-1,spectral-gcn-2} aim at defining convolution operation using Fourier transformation and graph Laplacian. These methods do not directly generalize to graphs with different structure \cite{spec-not-general}. Non-spectral methods \cite{deterministic-graph-cluster-1,gcn,gin,monet,gConv} define convolution through a local neighborhood around nodes in the graph. They are faster than spectral methods and easily generalize to other graphs. GNNs can also be viewed as \textit{message passing} algorithm where nodes iteratively aggregate messages from neighboring nodes through edges \cite{message_passing}.

\subsection{Pooling}
Pooling layers overcome GNN's inability to aggregate nodes hierarchically. Earlier pooling methods focused on deterministic graph clustering algorithms \cite{deterministic-graph-cluster-1,deterministic-graph-cluster-2,deterministic-graph-cluster-3}. \citeauthor{diffpool} introduced the first differentiable pooling operator which out-performed the previous deterministic methods. Since then, new data-driven pooling methods have been proposed; both spectral \cite{eigen-pool,graclus} and non-spectral \cite{diffpool,topk}. Spectral methods aim at capturing the graph topology using eigen-decomposition algorithms. However, due to higher computational requirement for spectral graph techniques, they are not easily scalable to large graphs. Hence, we focus on non-spectral methods.

Pooling methods can further be divided into global and hierarchical pooling layers. Global pooling summarize the entire graph in just one step. Set2Set \cite{set2set} finds the importance of each node in the graph through iterative content-based attention. Global-Attention \cite{glob-att} uses an attention mechanism to aggregate nodes in the graph. SortPool \cite{sortpool} summarizes the graph by concatenating few nodes after sorting them based on their features. Hierarchical pooling is used to capture the topological information of graphs. \textbf{DiffPool} forms a fixed number of clusters by aggregating nodes. It uses GNN to compute a dense soft assignment matrix, making it infea-
\begin{table}[!tbh]\
	\renewcommand{\arraystretch}{1.20}
	\centering
	\resizebox{\columnwidth}{!}{
		\begin{tabular}{lcccc}
			\toprule
			\textbf{Property} & \textbf{DiffPool} & \textbf{TopK} & \textbf{SAGPool} & \textbf{ASAP} \\
			\cmidrule{1-5}
			Sparse &  & \cmark & \cmark & \cmark \\
			Node Aggregation & \cmark &  &  & \cmark \\
			Soft Edge Weights & \cmark &  &  & \cmark \\
			Variable number of clusters &  & \cmark & \cmark & \cmark \\
			\bottomrule
		\end{tabular}
	}
	\caption{\label{tab:property} Properties desired in hierarchical pooling methods.}
\end{table}
	
\noindent sible for large graphs. \textbf{TopK} scores nodes based on a learnable projection vector and samples a fraction of high scoring nodes. It avoids node aggregation and computing soft assignment matrix to maintain the sparsity in graph operations. \textbf{SAGPool} improve upon TopK by using a GNN to consider the graph structure while scoring nodes. Since TopK and SAGPool do not aggregate nodes nor compute soft edge weights, they are unable to preserve node and edge information effectively. 

To address these limitations, we propose ASAP, which has all the desirable properties of hierarchical pooling without compromising on sparsity in graph operations. Please see Table. \ref{tab:property} for an overall comparison of hierarchical pooling methods. Further comparison discussions between hierarchical architectures are presented in Sec. \ref{sec:discussion_comp}.

\section{Preliminaries}
\subsection{Problem Statement}
Consider a graph $G(\mathcal{V}, \mathcal{E}, X)$ with $N = |\mathcal{V}|$ nodes and $|\mathcal{E}|$ edges. Each node $v_{i} \in \mathcal{V}$ has $d$-dimensional feature representation denoted by $x_i$. $X \in \mathbb{R}^{N \times d}$ denotes the node feature matrix and $A \in \mathbb{R}^{N \times N}$ represents the weighted adjacency matrix. The graph $G$ also has a label $y$ associated with it. Given a dataset $D = \{(G_{1}, y_{1}),(G_{2}, y_{2}),...\}$, the task of graph classification is to learn a mapping $f:\mathcal{G} \rightarrow \mathcal{Y}$, where $\mathcal{G}$ is the set of input graphs and $\mathcal{Y}$ is the set of labels associated with each graph. A pooled graph is denoted by $G^{p}(\mathcal{V}^{p}, \mathcal{E}^{p}, X^p)$ with node embedding matrix $X^{p}$ and its adjacency matrix as $A^{p}$.

\subsection{Graph Convolution Networks} We use Graph Convolution Network (GCN) \cite{gcn} for extracting discriminative features for graph classification. GCN is defined as:
\begin{equation}
\label{eq:gcn}
    X^{(l+1)} = \sigma(\hat{D}^{-\frac{1}{2}} \hat{A} \hat{D}^{\frac{1}{2}} X^{(l)} W^{(l)}),
\end{equation}
where $\hat{A} = A + I$ for self-loops, $\hat{D} = \sum_{j}\hat{A}_{i,j}$ and $W^{(l)} \in \mathbb{R}^{d \times f}$ is a learnable matrix for any layer $l$. We use the initial node feature matrix wherever provided, i.e., $X^{(0)} = X$.

\subsection{Self-Attention}
\label{sec:self_attn}
Self-attention is used to find the dependency of an input on itself \cite{SA-cheng,vaswani}. An alignment score $\alpha_{i,j}$ is computed to map the importance of candidates $c_{j}$ on target query $q_{i}$. In self-attention, target query $q_{i}$ and candidates $c_{j}$ are obtained from input entities $\boldsymbol{h}=\{h_{1},...,h_{n}\}$. Self-attention can be categorized as Token2Token and Source2Token based on the choice of target query $q$ \cite{disan}.

\subsubsection{Token2Token (T2T)} selects both the target and candidates from the input set $\boldsymbol{h}$. In the context of additive attention \cite{bahdanau}, $\alpha_{i,j}$ is computed as:
\begin{equation}
\label{eq:t2t-add}
    \alpha_{i, j} = softmax(\Vec{v}^{T}\sigma(W h_{i} \mathbin\Vert W h_{j})).
\end{equation}
where $\mathbin\Vert$ is the concatenation operator.

\subsubsection{Source2Token (S2T)} finds the importance of each candidate to a specific global task which cannot be represented by any single entity. $\alpha_{i,j}$ is computed by dropping the target query term. Eq. \eqref{eq:t2t-add} changes to the following:
\begin{equation}
\label{eq:s2t-add}
    \alpha_{i,j} = softmax(\Vec{v}^{T}\sigma(W h_{j})).
\end{equation}

\subsection{Receptive Field} 

We extend the concept of receptive field $RF$ from pooling operations in CNN  to GNN\footnote{Please refer to Appendix Sec. \ref{ssec:cnn-asap} for more details on similarity between pooling methods in CNN and ASAP.}. We define $RF^{node}$ of a pooling operator as the number of hops needed to cover all the nodes in the neighborhood that influence the representation of a particular output node. Similarly, $RF^{edge}$ of a pooling operator is defined as the number of hops needed to cover all the edges in the neighborhood that affect the representation of an edge in the pooled graph $\mathcal{G}^{p}$.


\section{ASAP: Proposed Method}
\label{sec:proposed-method}

In this section we describe the components of our proposed method ASAP. As shown in Fig. \ref{fig:asap}(b), ASAP initially considers all possible local clusters with a fixed receptive field for a given input graph. It then computes the cluster membership of the nodes using an attention mechanism. These clusters are then scored using a GNN as depicted in Fig \ref{fig:asap}(c). Further, a fraction of the top scoring clusters are selected as nodes in the pooled graph and new edge weights are computed between neighboring clusters as shown in Fig. \ref{fig:asap}(d). Below, we discuss the working of ASAP in details. Please refer to Appendix Sec. \ref{sec:pseudocode} for a pseudo code of the working of ASAP.

\subsection{Cluster Assignment}
\label{ssec:cluster-assignment}
Initially, we consider each node $v_{i}$ in the graph as a \textit{medoid} of a cluster $c_{h}(v_{i})$ such that each cluster can represent only the local neighbors $\mathcal{N}$ within a fixed radius of $h$ hops i.e., $c_{h}(v_{i}) = \mathcal{N}_h(v_{i})$. This effectively means that $RF^{node}=h$ for ASAP. This helps the clusters to effectively capture the information present in the graph sub-structure.


\noindent Let $x_{i}^{c}$ be the feature representation of a cluster $c_{h}(v_{i})$ centered at $v_i$. We define $G^{c}(\mathcal{V}, \mathcal{E}, X^c)$ as the graph with node feature matrix $X^{c} \in \mathbb{R}^{N \times d}$ and adjacency matrix $A^{c} = A$. We denote the cluster assignment matrix by $S \in \mathbb{R}^{N \times N}$, where $S_{i,j}$ represents the membership of node $v_{i} \in \mathcal{V}$ in cluster $c_{h}(v_{j})$. By employing such local clustering \cite{satu}, we can maintain sparsity of the cluster assignment matrix $S$ similar to the original graph adjacency matrix $A$ i.e., space complexity of both $S$ and $A$ is $\mathcal{O}(|\mathcal{E}|)$.

\subsection{Cluster Formation using Master2Token}
\label{ssec:node-aggregation}
Given a cluster $c_{h}(v_{i})$, we learn the cluster assignment matrix $S$ through a self-attention mechanism. The task here is to learn the overall representation of the cluster $c_{h}(v_{i})$ by attending to the relevant nodes in it.  We observe that both T2T and S2T attention mechanisms described in Sec. \ref{sec:self_attn} do not utilize any intra-cluster information. Hence, we propose a new variant of self-attention called \textbf{Master2Token (M2T)}. We further motivate the need for M2T framework later in Sec. \ref{ssec:m2t_compare}. In M2T framework, we first create a master query $m_{i} \in \mathbb{R}^{d}$ which is representative of all the nodes within a cluster:
\begin{equation}
    m_{i} = f_{m}(x_j' | v_j \in c_h(v_i) \}),
\end{equation}
where $x_j'$ is obtained after passing $x_j$ through a separate GCN to capture structural information in the cluster $c_{h}(v_{i})$ \footnote{If $x_j$ is used as it is then interchanging any two nodes in a cluster will have not affect the final output, which is undesirable.}. $f_{m}$ is a master function which combines and transforms feature representation of $v_{j}\in c_{h}(v_{i})$ to find $m_i$. In this work we experiment with $max$ master function defined as:

\begin{equation}
    m_i = \max_{v_{j} \in c_{h}(v_{i})}(x_{j}').
\end{equation}
This master query $m_{i}$ attends to all the constituent nodes $v_{j} \in c_{h}(v_{i})$ using additive attention:

\begin{equation}
\label{eq:m2t-add}
\alpha_{i, j} = softmax(\Vec{w}^{T}\sigma(W m_{i} \mathbin\Vert x_{j}')).
\end{equation}

where $\Vec{w}^{T}$ and $W$ are learnable vector and matrix respectively. The calculated attention scores $\alpha_{i,j}$ signifies the membership strength of node $v_{j}$ in cluster $c_{h}(v_{i})$. Hence, we use this score to define the cluster assignment matrix discussed above, i.e., $S_{i,j} = \alpha_{i,j}$. The cluster representation $x_{i}^{c}$ for ${c_{h}(v_i)}$ is computed as follows:
\begin{equation}
\label{eq:cluster_repr}
    x_{i}^{c} = \sum_{j=1}^{|c_{h}(v_{i})|} \alpha_{i,j} x_{j}.
\end{equation}

\subsection{Cluster Selection using LEConv}
\label{ssec:cluster-selection}
Similar to TopK \cite{topk}, we sample clusters based on a cluster fitness score $\phi_{i}$ calculated for each cluster in the graph $G^c$ using a fitness function $f_{\phi}$. For a given pooling ratio $k \in (0, 1]$, the top $\lceil kN\rceil$ clusters are selected and included in the pooled graph $G^p$. To compute the fitness scores, we introduce \textbf{Local Extrema Convolution (LEConv)}, a graph convolution method which can capture local extremum information. In Sec. \ref{ssec:leconv-use} we motivate the choice of LEConv's formulation and contrast it with the standard GCN formulation. LEConv is used to compute $\phi$ as follows:
\begin{equation}
\label{eq:leconv}
    \phi_{i} = \sigma(x_{i}^c W_{1} + \sum_{j \in \mathcal{N}(i)} A_{i,j}^c (x_{i}^c W_{2} - x_{j}^c W_{3}))
\end{equation}
where $\mathcal{N}(i)$ denotes the neighborhood of the $i^{th}$ node in $G^c$. $W_{1}, W_{2}, W_{3}$ are learnable parameters and $\sigma(.)$ is some activation function. Fitness vector ${\Phi} = [\phi_{1}, \phi_{2},...,\phi_{N}]^{T}$ is multiplied to the cluster feature matrix $X^{c}$ to make $f_{\phi}$ learnable i.e.,:
\begin{equation}
    \nonumber
    \hat{X^{c}} = {\Phi} \odot X^{c},
\end{equation}
where $\odot$ is broadcasted hadamard product. The function $\text{TOP}_k(.)$ ranks the fitness scores and gives the indices $\hat{i}$ of top $\lceil kN \rceil$ selected clusters in $G^{c}$ as follows:
\begin{equation}
\nonumber
    \hat{i} = \text{TOP}_k(\hat{X^{c}}, \lceil kN \rceil).
\end{equation}
The pooled graph $G^p$ is formed by selecting these top $\lceil kN \rceil$ clusters. The pruned cluster assignment matrix $\hat{S} \in \mathbb{R}^{N \times \lceil kN \rceil}$ and the node feature matrix $X^{p} \in \mathbb{R}^{\lceil kN \rceil \times d}$ are given by:
\begin{equation}
\label{eq:x-pool}   
    \hat{S} = S(:, \hat{i}), \hspace{1cm} X^{p} = \hat{X}^{c}(\hat{i}, :)
\end{equation}
where $\hat{i}$ is used for index slicing.

\begin{table*}[tbh]\
	\centering
	\resizebox{\textwidth}{!}{
		\begin{tabular}{lccccc}
			\toprule
			Method & \textsc{D\&D} & \textsc{PROTEINS} & \textsc{NCI1} & \textsc{NCI109} & \textsc{FRANKENSTEIN}\\
			\midrule
			\textsc{Set2Set} \cite{set2set}  			& 71.60 $\pm$ 0.87 & 72.16 $\pm$ 0.43 & 66.97 $\pm$ 0.74 & 61.04 $\pm$ 2.69 & 61.46 $\pm$ 0.47\\
			\textsc{Global-Attention} \cite{glob-att}  	& 71.38 $\pm$ 0.78 & 71.87 $\pm$ 0.60 & 69.00 $\pm$ 0.49 & 67.87 $\pm$ 0.40 & 61.31 $\pm$ 0.41\\
			\textsc{SortPool} \cite{sortpool}  			& 71.87 $\pm$ 0.96 & 73.91 $\pm$ 0.72 & 68.74 $\pm$ 1.07 & 68.59 $\pm$ 0.67 & 63.44 $\pm$ 0.65\\
			\midrule
			\textsc{Diffpool} \cite{diffpool}  			& 66.95 $\pm$ 2.41 & 68.20 $\pm$ 2.02 & 62.32 $\pm$ 1.90 & 61.98 $\pm$ 1.98 & 60.60 $\pm$ 1.62\\
			\textsc{TopK} \cite{topk} 					& 75.01 $\pm$ 0.86 & 71.10 $\pm$ 0.90 & 67.02 $\pm$ 2.25 & 66.12 $\pm$ 1.60 & 61.46 $\pm$ 0.84\\
			\textsc{SAGPool} \cite{sag}  				& 76.45 $\pm$ 0.97 & 71.86 $\pm$ 0.97 & 67.45 $\pm$ 1.11  & 67.86 $\pm$ 1.41 & 61.73 $\pm$ 0.76\\
			\midrule
			\textsc{ASAP} (Ours) & $\mathbf{76.87 \pm 0.7}$ & $\mathbf{74.19 \pm 0.79}$ & $\mathbf{71.48 \pm 0.42}$ & $\mathbf{70.07 \pm 0.55}$ & $\mathbf{66.26 \pm 0.47}$\\
			\bottomrule
		\end{tabular}
	}
	\caption{\label{tab:comparison} Comparison of ASAP with previous global and hierarchical pooling. Average accuracy and standard deviation is reported for 20 random seeds. We observe that ASAP consistently outperforms all the baselines on all the datasets. Please refer to Sec. \ref{sec:results} for more details.}
\end{table*}


\subsection{Maintaining Graph Connectivity}
\label{ssec:graph-connectivity}
Following \cite{diffpool}, once the clusters have been sampled, we find the new adjacency matrix $A^{p}$ for the pooled graph $G^p$ using $\hat{A}^c$ and $\hat{S}$ in the following manner:
\begin{equation}
\label{eq:stas}
    A^{p} = \hat{S}^{T} \hat{A}^c \hat{S}
\end{equation}
where $\hat{A}^c = A^c + I$. Equivalently, we can see that $A^{p}_{i,j} = \sum_{k,l} \hat{S}_{k,i} \hat{A}^c_{k,l} \hat{S}_{l,j}$. This formulation ensures that any two clusters $i$ and $j$ in $G^{p}$ are connected if there is any common node in the clusters $c_{h}(v_{i})$ and $c_{h}(v_{j})$ or if any of the constituent nodes in the clusters are neighbors in the original graph $G$ (Fig. \ref{fig:asap}(d)). Hence, the strength of the connection between clusters is determined by both the membership of the constituent nodes through $\hat{S}$ and the edge weights $A^c$. Note that $\hat{S}$ is a sparse matrix by formulation and hence the above operation can be implemented efficiently.

\section{Theoretical Analysis}
\subsection{Limitations of using GCN for scoring clusters}

\label{ssec:leconv-use}
GCN from Eq. \eqref{eq:gcn} can be viewed as an operator which first computes a \textit{pre-score} $\hat{\phi^{\prime}}$ for each node i.e., $\hat{\phi^{\prime}} = XW$ followed by a weighted average over neighbors and a non-linearity. If for some node the pre-score is very high, it can increase the scores of its neighbors which inherently biases the pooling operator to select clusters in the local neighborhood instead of sampling clusters which represent the whole graph. Thus, selecting the clusters which correspond to local extremas of pre-score function would potentially allow us to sample representative clusters from all parts of the graph. 



\begin{theorem}
	\label{thm:gcn-score}
	Let $\mathcal{G}$ be a graph with positive adjacency matrix A i.e., $A_{i, j}\geq0$. Consider any function $f(X, A): \mathbb{R}^{N \times d} \times \mathbb{R}^{N \times N} \rightarrow \mathbb{R}^{N \times 1}$ which depends on difference between a node and its neighbors after a linear transformation $W \in \mathbb{R}^{d \times 1}$. For e.g,:
	\begin{equation}
	\nonumber
	f_{i} = \sigma(\alpha_{i}x_{i}W + \sum_{j \in \mathcal{N}(i)} \beta_{i,j} (x_{i}W - x_{j}W))
	\end{equation}
	 where $f_{i}, \alpha_{i}, \beta_{i,j} \in \mathbb{R}$ and $ x_{i} \in \mathbb{R}^{d}$.
	
	\begin{enumerate}[label=\alph*)]
		\item If fitness value $\Phi = GCN(X, A)$ then $\Phi$ cannot learn f.
		\item If fitness value $\Phi = LEConv(X, A)$ then $\Phi$ can learn f.
	\end{enumerate}
\end{theorem}

\begin{proof}
See Appendix Sec. \ref{ssec:gconv-proof} for proof.
\end{proof}
\noindent Motivated by the above analysis, we propose to use LEConv (Eq. \ref{eq:leconv}) for scoring clusters. LEConv can learn to score clusters by considering both its global and local importance through the use of self-loops and ability to learn functions of local extremas.

\subsection{Graph Connectivity}
Here, we analyze ASAP from the aspect of edge connectivity in the pooled graph. When considering $h$-hop neighborhood for clustering, both ASAP and DiffPool have $RF^{edge} = 2h+1$ because they use Eq. \eqref{eq:stas} to define the edge connectivity. On the other hand, both TopK and SAGPool have $RF^{edge} = h$. A larger edge receptive field implies that the pooled graph has better connectivity which is important for the flow of information in the subsequent GCN layers.

\begin{theorem}

\label{thm:star-graph}
	Let the input graph $\mathcal{G}$ be a tree of any possible structure with $N$ nodes. Let $k^{*}$ be the lower bound on sampling ratio $k$ to ensure the existence of atleast one edge in the pooled graph irrespective of the structure of $\mathcal{G}$ and the location of the selected nodes. For TopK or SAGPool, $k^{*} \rightarrow 1$ whereas for ASAP, $k^{*} \rightarrow 0.5$ as $N \rightarrow \infty$.
\begin{proof}
See Appendix Sec. \ref{ssec:graph-connect-proof} for proof.
\end{proof}
\end{theorem}


\noindent Theorem \ref{thm:star-graph} suggests that ASAP can achieve a similar degree of connectivity as SAGPool or TopK for a much smaller sampling ratio $k$. For a tree with no prior information about its structure, ASAP would need to sample only half of the clusters whereas TopK and SAGPool would need to sample almost all the nodes, making TopK and SAGPool inefficient for such graphs. In general, independent of any combination of nodes selected, ASAP will have better connectivity due to its larger receptive field. Please refer to Appendix Sec. \ref{ssec:graph-connect-proof} for a similar analysis on path graph and more details.

\subsection{Graph Permutation Equivariance}

\begin{proposition}
ASAP is a graph permutation equivariant pooling operator. 
\end{proposition}
\begin{proof}
See Appendix Sec. \ref{ssec:perm-eq-proof} for proof.
\end{proof}

\section{Experimental Setup}
In our experiments, we use $5$ graph classification benchmarks and compare ASAP with multiple pooling methods. Below, we describe the statistics of the dataset, the baselines used for comparisons and our evaluation setup in detail.

\subsection{Datasets}
We demonstrate the effectiveness of our approach on $5$ graph classification datasets. D\&D \cite{dd1,dd2-proteins} and PROTEINS \cite{dd2-proteins,proteins} are datasets containing proteins as graphs. NCI1 \cite{nci1} and NCI109 are datasets for anticancer activity classification. FRANKENSTEIN \cite{frankenstein} contains molecules as graph for mutagen classification. Please refer to Table \ref{tab:stats} for the dataset statistics.

\begin{table}[!tbh]
	\centering
	\resizebox{1\columnwidth}{!}{
		\begin{tabular}{lcccc}
			\toprule
			\textbf{Dataset} & $\text{G}_{avg}$ & $\text{C}_{avg}$ & $\text{V}_{avg}$ & $\text{E}_{avg}$ \\
			\midrule
			D\&D    & 1178 & 2 & 284.32 & 715.66 \\
			PROTEINS    & 1113 & 2 & 39.06 & 72.82 \\
			NCI1    & 4110 & 2 & 29.87 & 32.30 \\
			NCI109    & 4127 & 2 & 29.68 & 32.13 \\
			FRANKENSTEIN    & 4337 & 2 & 16.90 & 17.88 \\
			\bottomrule
		\end{tabular}
	}
	\caption{\label{tab:stats} Statistics of the graph datasets. $\text{G}_{avg}$, $\text{C}_{avg}$, $\text{V}_{avg}$ and $\text{E}_{avg}$ denotes the average number of graphs, classes, nodes and edges respectively.}
\end{table}

\subsection{Baselines}
We compare ASAP with previous state-of-the-art hierarchical pooling operators DiffPool \cite{diffpool}, TopK \cite{topk} and SAGPool \cite{sag}. For comparison with global pooling, we choose Set2Set \cite{set2set}, Global-Attention \cite{glob-att} and SortPool \cite{sortpool}.

\subsection{Training \& Evaluation Setup}
We use a similar architecture as defined in \cite{topk2,sag} which is depicted in Fig. \ref{fig:asap}(f). For ASAP, we choose $k = 0.5$ and $h = 1$ to be consistent with baselines.\footnote{Please refer to Appendix Sec. \ref{ssec:hyper-tune} for further details on hyperparameter tuning and Appendix Sec. \ref{ssec:ablation-k} for ablation on $k$.} Following SAGPool\cite{sag}, we conduct our experiments using $10$-fold cross-validation and report the average accuracy on $20$ random seeds.\footnote{Source code for ASAP can be found at: \url{https://github.com/malllabiisc/ASAP}}

\begin{table}[!tbh]\
	\small
	\centering
	\begin{tabular}{lcc}
		\toprule
		Aggregation type & FITNESS & CLUSTER \\
		\midrule
		None 		& - 		& - 		\\
		Only cluster 		& - 		& \cmark 	\\
		Both 		& \cmark 	& \cmark 	\\
		\bottomrule
	\end{tabular}
	\caption{\label{tab:aggr_types} Different aggregation types as mentioned in Sec \ref{sec:ablation_aggr}.}
\end{table}

\section{Results}
In this section, we attempt to answer the following questions:
\begin{description}
	\item[Q1] How does ASAP perform compared to other pooling methods at the task of graph classification? (Sec. \ref{sec:results})
	\item[Q2] Is cluster formation by M2T attention based node aggregation beneficial during pooling? (Sec. \ref{sec:ablation_attn})
	\item[Q3] Is LEConv better suited as cluster fitness scoring function compared to vanilla GCN? (Sec. \ref{sec:ablation_fitness})
	\item[Q4] How helpful is the computation of inter-cluster soft edge weights instead of sampling edges from the input graph? (Sec. \ref{sec:ablation_edge})
\end{description}

\subsection{Performance Comparison}
\label{sec:results}
We compare the performace of ASAP with baseline methods on $5$ graph classification tasks. The results are shown in Table \ref{tab:comparison}. All the numbers for hierarchical pooling (DiffPool, TopK and SAGPool) are taken from \cite{sag}. For global pooling (Set2Set, Global-Attention and SortPool), we modify the architectural setup to make them comparable with the hierarchical variants. \footnote{Please refer to Appendix Sec. \ref{ssec:global-pool} for more details}. We observe that ASAP consistently outperforms all the baselines on all $5$ datasets. We note that ASAP has an average improvement of $4\%$ and $3.5\%$ over previous state-of-the-art hierarchical (SAGPool) and global (SortPool) pooling methods respectively. We also observe that compared to other hierarchical methods, ASAP has a smaller variance in performance which suggests that the training of ASAP is more stable.

\subsection{Effect of Node Aggregation}
\label{sec:ablation_aggr}
Here, we evaluate the improvement in performance due to our proposed technique of aggregating nodes to form a cluster. There are two aspects involved during the creation of clusters for a pooled graph:

\begin{itemize}
	\item FITNESS: calculating fitness scores for individual nodes. Scores can be calculated either by using only the medoid or by aggregating neighborhood information.
	\item CLUSTER: generating a representation for the new cluster node. Cluster representation can either be the medoid's representation or some feature aggregation of the neighborhood around the medoid.
\end{itemize}

\noindent We test three types of aggregation methods: 'None', 'Only cluster' and 'Both' as described in Table \ref{tab:aggr_types}. As shown in Table \ref{tab:aggregation}, we observe that our proposed node aggregation helps improve the performance of ASAP.

\begin{table}[!tbh]\
	\centering
	\begin{tabular}{lcc}
		\toprule
		Aggregation & \textsc{FRANKENSTEIN} & \textsc{NCI1} \\
		\midrule
		None & 67.4 $\pm$0.6 & 69.9 $\pm$ 2.5\\
		Only cluster & 67.5 $\pm$0.5 & 70.6 $\pm$ 1.8\\
		Both & $\mathbf{67.8 \pm 0.6}$ & $\mathbf{70.7 \pm 2.3}$ \\
		\bottomrule
	\end{tabular}
	\caption{\label{tab:aggregation} Performace comparison of different aggregation methods on validation data of FRANKENSTEIN and NCI1.}
\end{table}

\begin{table}[!tbh]\
	\centering
	\begin{tabular}{lcc}
		\toprule
		Attention & \textsc{FRANKENSTEIN} & \textsc{NCI1} \\
		\midrule
		T2T & 67.6 $\pm$ 0.5 & 70.3 $\pm$ 2.0 \\
		S2T & 67.7 $\pm$ 0.5 & 69.9 $\pm$ 2.0 \\
		M2T & $\mathbf{67.8 \pm 0.6}$ & $\mathbf{70.7 \pm 2.3}$ \\
		\bottomrule
	\end{tabular}
	\caption{\label{tab:attention} Effect of different attention framework on pooling evaluated on validation data of FRANKENSTEIN and NCI1. Please refer to Sec. \ref{sec:ablation_attn} for more details.}
\end{table}

\subsection{Effect of M2T Attention}
\label{sec:ablation_attn}
We compare our M2T attention framework with previously proposed S2T and T2T attention techniques. The results are shown in Table \ref{tab:attention}. We find that M2T attention is indeed better than the rest in NCI1 and comparable in FRANKENSTEIN.

\begin{table}[!ht]\
	\centering
	\begin{tabular}{lcc}
		\toprule
		Fitness function & \textsc{FRANKENSTEIN} & \textsc{NCI1} \\
		\midrule
		GCN & 62.7$\pm$0.3 & 65.4$\pm$2.5 \\
		Basic-LEConv & 63.1$\pm$0.7 & 69.8$\pm$1.9 \\
		LEConv & \textbf{67.8$\pm$0.6} & \textbf{70.7$\pm$2.3} \\
		\bottomrule
	\end{tabular}
	\caption{\label{tab:leconv} Performance comparison of different fitness scoring functions on validation data of FRANKENSTEIN and NCI1. Refer to Sec. \ref{sec:ablation_fitness} for details.}
\end{table}

\subsection{Effect of LEConv as a fitness scoring function}
\label{sec:ablation_fitness}
In this section, we analyze the impact of LEConv as a fitness scoring function in ASAP. We use two baselines - GCN (Eq. \ref{eq:gcn}) and Basic-LEConv which computes $\phi_{i} = \sigma(x_{i}W + \sum_{j \in \mathcal{N}(x_{i})} A_{i, j} (x_{i}W-x_{j}W))$. In Table \ref{tab:leconv} we can see that Basic-LEConv and LEConv perform significantly better than GCN because of their ability to model functions of local extremas. Further, we observe that LEConv performs better than Basic-LEConv as it has three different linear transformation compared to only one in the latter. This allows LEConv to potentially learn complicated scoring functions which is better suited for the final task. Hence, our analysis in Theorem \ref{thm:gcn-score} is emperically validated.

\subsection{Effect of computing Soft edge weights}
\label{sec:ablation_edge}
We evaluate the importance of calculating edge weights for the pooled graph as defined in Eq. \ref{eq:stas}. We use the best model configuration as found from above ablation analysis and then add the feature of computing soft edge weights for clusters. We observe a significant drop in performace when the edge weights are not computed. This proves the necessity of capturing the edge information while pooling graphs.

\begin{table}[!tbh]\
	\centering
	\begin{tabular}{ccc}
		\toprule
		Soft edge weights & \textsc{FRANKENSTEIN} & \textsc{NCI1} \\
		\midrule
		Absent & 67.8 $\pm$ 0.6 & 70.7 $\pm$ 2.3 \\
		Present & $\mathbf{68.3 \pm 0.5}$ & $\mathbf{73.4 \pm 0.4}$\\
		\bottomrule
	\end{tabular}
	\caption{\label{tab:stas} Effect of calculating soft edge weights on pooling for validation data of FRANKENSTEIN and NCI1. Please refer to Sec. \ref{sec:ablation_edge} for more details.}
\end{table}

\section{Discussion}

\subsection{Comparison with other pooling methods}
\label{sec:discussion_comp}
\subsubsection{DiffPool}
DiffPool and ASAP both aggregate nodes to form a cluster. While ASAP only considers nodes which are within $h$-hop neighborhood from a node $x_{i}$ (medoid) as a cluster, DiffPool considers the entire graph. As a result, in DiffPool, two nodes that are disconnected or far away in the graph can be assigned similar clusters if the nodes and their neighbors have similar features. Since this type of cluster formation is undesirable for a pooling operator \cite{diffpool}, DiffPool utilizes an auxiliary link prediction objective during training to specifically prevent far away nodes from being clustered together. ASAP needs no such additional regularization because it ensures the localness while clustering. DiffPool's soft cluster assignment matrix $S$ is calculated for all the nodes to all the clusters making $S$ a dense matrix. Calculating and storing this does not scale easily for large graphs. ASAP, due to the local clustering over $h$-hop neighborhood, generates a sparse assignment matrix while retaining the hierarchical clustering properties of Diffpool. Further, for each pooling layer, DiffPool has to predetermine the number of clusters it needs to pick which is fixed irrespective of the input graph size. Since ASAP selects the top $k$ fraction of nodes in current graph, it inherently takes the size of the input graph into consideration.

\subsubsection{TopK \& SAGPool}
While TopK completely ignores the graph structure during pooling, SAGPool modifies the TopK formulation by incorporating the graph structure through the use of a GCN network for computing node scores $\phi$. To enforce sparsity, both TopK and SAGPool avoid computing the cluster assignment matrix $S$ that DiffPool proposed. Instead of grouping multiple nodes to form a cluster in the pooled graph, they \textit{drop} nodes from the original graph   based on a score \cite{topk2} which might potentially lead to loss of node and edge information. Thus, they fail to leverage the overall graph structure while creating the clusters. In contrast to TopK and SAGPool, ASAP can capture the rich graph structure while aggregating nodes to form clusters in the pooled graph. TopK and SAGPool sample edges from the original graph to define the edge connectivity in the pooled graph. Therefore, they need to sample nodes from a local neighborhood to avoid isolated nodes in the pooled graph. Maintaining graph connectivity prevents these pooling operations from sampling representative nodes from the entire graph. The pooled graph in ASAP has a better edge connectivity compared to TopK and SAGPool because soft edge weights are computed between clusters using upto three hop connections in the original graph. Also, the use of LEConv instead of GCN for finding fitness values $\phi$ further allows ASAP to sample representative clusters from local neighborhoods over the entire graph.

\subsection{Comparison of Self-Attention variants}
\label{ssec:m2t_compare}
\subsubsection{Source2Token \& Token2Token}
T2T models the membership of a node by generating a query based only on the medoid of the cluster. Graph Attention Network (GAT) \cite{gat} is an example of T2T attention in graphs. S2T finds the importance of each node for a global task. As shown in Eq. \ref{eq:s2t-add}, since a query vector is not used for calculating the attention scores, S2T inherently assigns the same membership score to a node for all the possible clusters that node can belong to. Hence, both S2T and T2T mechanisms fail to effectively utilize the intra-cluster information while calculating a node's cluster membership. On the other hand, M2T uses a master function $f_{m}$ to generate a query vector which depends on all the entities within the cluster and hence is a more representative formulation. To understand this, consider the following scenario. If in a given cluster, a non-medoid node is removed, then the un-normalized membership scores for the rest of the nodes will remain unaffected in S2T and T2T framework whereas the change will reflect in the scores calculated using M2T mechanism. Also, from Table \ref{tab:attention}, we find that M2T performs better than S2T and T2T attention showing that M2T is better suited for global tasks like pooling.


\section{Conclusion}
In this paper, we introduce ASAP, a sparse and differentiable pooling method for graph structured data. ASAP clusters local subgraphs hierarchically which helps it to effectively learn the rich information present in the graph structure. We propose Master2Token self-attention framework which enables our model to better capture the membership of each node in a cluster. We also propose LEConv, a novel GNN formulation that scores the clusters based on its local and global importance. ASAP leverages LEConv to compute cluster fitness scores and samples the clusters based on it. This ensures the selection of representative clusters throughout the graph. ASAP also calculates sparse edge weights for the selected clusters and is able to capture the edge connectivity information efficiently while being scalable to large graphs. We validate the effectiveness of the components of ASAP both theoretically and empirically. Through extensive experiments, we demonstrate that ASAP achieves state-of-the-art performace on multiple graph classification datasets.



\section{Acknowledgements}
We would like to thank the developers of Pytorch\_Geometric \cite{bhai} which allows quick implementation of geometric deep learning models. We would like to thank Matthias Fey again for actively maintaining the library and quickly responding to our queries on github.

\bibliographystyle{aaai}
{\fontsize{9.0pt}{10.0pt} \selectfont \bibliography{ms}}


\appendix

\setcounter{secnumdepth}{1}
	
	\vspace{1cm}
	{\hspace{2.7cm} \textbf{\LARGE Appendix}}

	\section{Hyperparameter Tuning}
	\label{ssec:hyper-tune}
	For all our experiments, Adam \cite{adam} optimizer is used. $10$-fold cross-validation is used with $80\%$ for training and $10\%$ for validation and test each. Models were trained for $100$ epochs with lr decay of $0.5$ after every $50$ epochs. The range of hyperparameter search are provided in Table \ref{tab:hyper-tune}. The model with best validation accuracy was selected for testing. Our code is based on Pytorch Geometric library \cite{bhai}.

	\begin{table}[tbh!]
		\begin{center}
			\begin{small}
				\begin{tabular}{lcl}
					\toprule
					\textbf{Hyperparameter} & & \textbf{Range} \\
					\cmidrule{1-3}
					Hidden dimension & & $\{16, 32, 64, 128\}$ \\
					Learning rate & & $\{0.01, 0.001\}$ \\
					Dropout & & $\{0, 0.1, 0.2, 0.3, 0.4, 0.5\}$\\
					L2 regularization & & $5e^{-4}$ \\
					\bottomrule
				\end{tabular}
				\caption{\label{tab:hyper-tune} Hyperparameter tuning Summary.}
			\end{small}
		\end{center}
	\end{table}

	\section{Details of Hierarchical Pooling Setup}
	\label{ssec:global-pool}
	For hierarchical pooling, we follow SAGPool \cite{sag} and use three layers of GCN, each followed by a pooling layer. After each pooling step, the graph is summarized using a readout function which is a concatenation of the $mean$ and $max$ of the node representations (similar to SAGPool). The summaries are then added and passed through a network of fully-connected layers separated by dropout layers to predict the class.

	\section{Details of Global Pooling Setup}
	Global Pooling architecture is same as the hierarchical architecture with the only difference that pooling is done only after all GCN layers. We do not use readout function for global pooling as they do not require them. To be comparable with other models, we restrict the feature dimension of the pooling output to be no more than $256$. For global pooling layers, range for hidden dimension and lr search was same as ASAP.
	
	\begin{table}[tbh!]
		\begin{center}
			\begin{small}
				\resizebox{\columnwidth}{!}{
					\begin{tabular}{lcl}
						\toprule
						\textbf{Method} & & \textbf{Range} \\
						\cmidrule{1-3}
						Set2Set & & processing-step $\in \{5, 10\}$ \\
						Global-Attention & & transform $\in \{True, False\}$ \\
						SortPool & & $K$ is chosen such that output of pooling $\leq 256$\\
						\bottomrule
					\end{tabular}
				}
				\caption{\label{tab:statistics}Global Pooling Hyperparameter Tuning Summary.}
			\end{small}
		\end{center}
		\vskip -0.1in
	\end{table}

	
	\section{Similarities between pooling in CNN and ASAP}
	\label{ssec:cnn-asap}
	In CNN, pooling methods (e.g mean pool and max pool) have two hyperparameter: kernel size and stride. Kernel size decides the number of pixels being considered for computing each new pixel value in the next layer. Stride decides the fraction of new pixels being sampled thereby controlling the size of the image in next layer. In ASAP, $RF^{node}$ determines the neighborhood radius of clusters and $k$ decides the sampling ratio. This makes $RF^{node}$ and $k$ are analogous to kernel size and stride of CNN pooling respectively. There are however some key differences. In CNN, a given kernel size corresponds to a fixed number of pixels around a central pixel whereas in ASAP, the number of nodes being considered is variable, although the neighborhood $RF^{node}$ is constant. In CNN, stride uniformly samples from new pixels whereas in ASAP, the model has the flexibility to attend to different parts of the graph and sample accordingly.

	\section{Ablation on pooling ratio $k$}
	\label{ssec:ablation-k}
	
	\begin{figure}[!ht]
		\centering
		\includegraphics[width=0.4\textwidth]{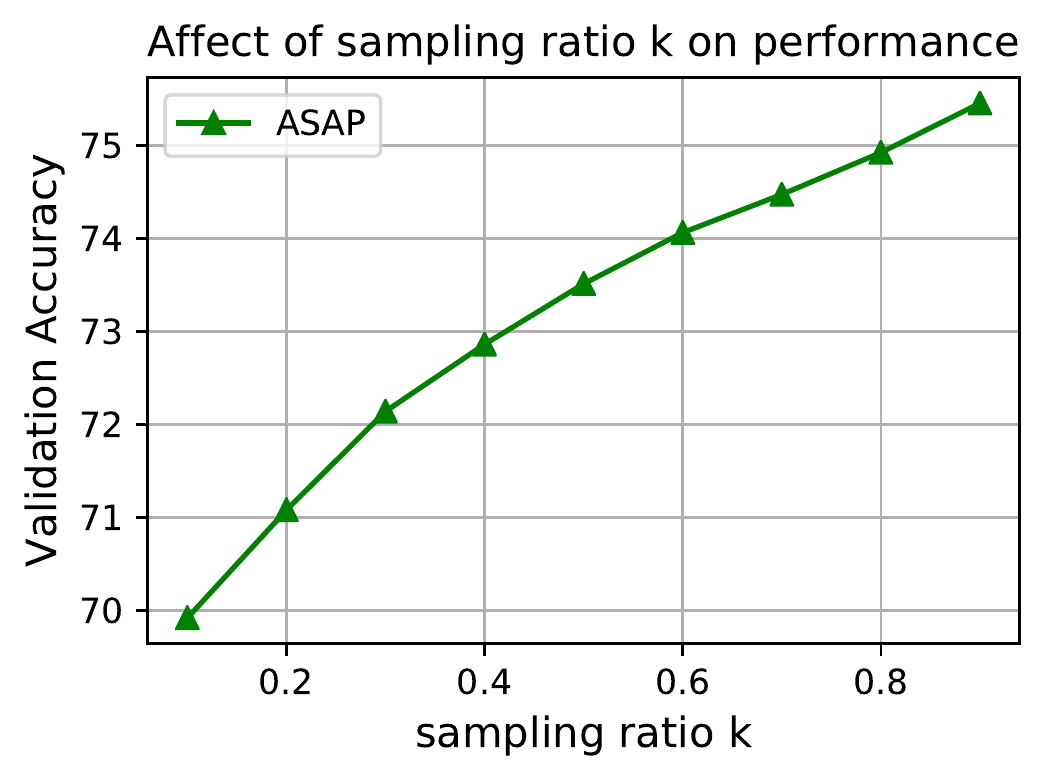}
		\caption{\label{fig:sampling-ratio-k} Validation Accuracy vs sampling ratio $k$ on NCI1 dataset.}
	\end{figure}
	
	Intuitively, higher $k$ will lead to more information retention. Hence, we expect an increase in performance with increasing $k$. This is empirically observed in Fig. \ref{fig:sampling-ratio-k}. However, as $k$ increases, the computational resources required by the model also increase because a relatively larger pooled graph gets propagated to the later layers. Hence, there is a trade-off between performance and computational requirement while deciding on the pooling ratio $k$.

	\section{Proof of Theorem 1}
	\label{ssec:gconv-proof}
	\textbf{Theorem 1.}
	\textit{
		Let $\mathcal{G}$ be a graph with positive adjacency matrix A i.e., $A_{i, j}\geq0$. Consider any function $f(X, A): \mathbb{R}^{N \times d} \times \mathbb{R}^{N \times N} \rightarrow \mathbb{R}^{N \times 1}$ which depends on difference between a node and its neighbors after a linear transformation $W \in \mathbb{R}^{d \times d}$. For e.g:}
	\begin{equation}
	\nonumber
	f_{i} = \sigma(\alpha_{i}x_{i}W + \sum_{j \in \mathcal{N}(i)} \beta_{i,j} (x_{i}W - x_{j}W))
	\end{equation}
	\textit{
		where $f_{i}, \alpha_{i}, \beta_{i,j} \in \mathbb{R}$ and $ x_{i} \in \mathbb{R}^{d}$.
	}
	\begin{enumerate}[label=\alph*)]
		\item \textit{If fitness value $\phi = GCN(X, A)$ then $\phi$ cannot learn f.}
		\item \textit{If fitness value $\phi = LEConv(X, A)$ then $\phi$ can learn f.}
	\end{enumerate}
	
	\begin{proof}
		For GCN, $\phi_{i} = \sigma(\hspace{1mm} \sum_{j \in \mathcal{N}(x_{i}) \cup \{i\}} A_{i, j} x_{j}W)$ where $W$ is a learnable matrix. Since $A_{i, j} \geq 0$, $\phi_{i}$ cannot have a term of the form $\beta_{i,j} (x_{i}W - x_{j}W)$ which proves the first part of the theorem. We prove the second part by showing that LEConv can learn the following function $f$: 
		\begin{equation}
		\label{eq:to-prove-f}
		f_{i} = \sigma(\alpha_{i}x_{i}W + \sum_{j \in \mathcal{N}(i)} \beta_{i,j} (x_{i}W - x_{j}W))
		\end{equation}
		LEConv formulation is defined as:
		\begin{equation}
		\label{eq:club}
		\phi_{i} = \sigma(x_{i} W_{1} + \sum_{j \in \mathcal{N}(i)} A_{i,j} (x_{i} W_{2} - x_{j} W_{3}))
		\end{equation}
		where $W_{1}$, $W_{2}$ and $W_{3}$ are learnable matrices. For $W_{3} = W_{2} = W_{1}$, $\alpha_{1} = 1$ and $\beta_{i,j} = A_{i,j}$  we find Eq. \eqref{eq:club} is equal to Eq. \eqref{eq:to-prove-f}.
	\end{proof}

	\section{Graph Connectivity}
	\subsection{\Large Proof of Theorem 2}
	\label{ssec:graph-connect-proof}
	
	\begin{definition}
		For a graph $\mathcal{G}$, we define optimum-nodes $n^{*}_{h}(\mathcal{G})$ as the maximum number of nodes that can be selected which are atleast $h$ hops away from each other.
	\end{definition}
	
	\begin{definition}
		For a given number of nodes $N$, we define optimum-tree $\mathcal{T}^{*}_{N}$ as the tree which has maximum optimum-nodes $n^{*}_{h}(\mathcal{T}_{N})$ among all possible trees $\mathcal{T}_{N}$ with $N$ nodes.
	\end{definition}

	\begin{lemma}
		\label{lem:atmost-one}
		Let $\mathcal{T}^{*}_{N}$ be an optimum-tree of $N$ vertices and $\mathcal{T}^{*}_{N-1}$ be an optimum tree with $N-1$ vertices. The optimum-nodes of $\mathcal{T}^{*}_{N}$ and $\mathcal{T}^{*}_{N-1}$ differ by atmost one, i.e., $0 \le n^{*}_{h}(\mathcal{T}^{*}_{N}) - n^{*}_{h}(\mathcal{T}^{*}_{N-1}) \le 1$.
		\begin{proof}
			Consider $\mathcal{T}^{*}_{N}$ which has $N$ nodes.
			We can remove any one of the leaf nodes in $\mathcal{T}^{*}_{N}$ to obtain a tree $\mathcal{T}_{N-1}$ with $N-1$ nodes. 
			If any one of the nodes in $n^{*}_{h}(\mathcal{T}^{*}_{N})$ 
			was removed, then $n^{*}_{h}(\mathcal{T}_{N-1})$ would become 
			$n^{*}_{h}(\mathcal{T}^{*}_{N}) - 1$. If any other node was removed , then being a leaf it does not constitute the shortest path between any of the $n^{*}_{h}(\mathcal{T}_{N-1})$ nodes. This implies that the optimum-nodes for $\mathcal{T}_{N-1}$ is atleast $n^{*}_{h}(\mathcal{T}^{*}_{N})-1$, i.e.,
			\begin{equation}
			\label{eq:n-1-range}
			n^{*}_{h}(\mathcal{T}^{*}_{N})-1 \le n^{*}_{h}(\mathcal{T}_{N-1}) \le n^{*}_{h}(\mathcal{T}^{*}_{N})
			\end{equation}
			Since $\mathcal{T}^{*}_{N-1}$ is the optimal-tree, we know that:
			\begin{equation}
			\label{eq:optimum}
			n^{*}_{h}(\mathcal{T}_{N-1}) \le n^{*}_{h}(\mathcal{T}^{*}_{N-1})
			\end{equation}
			Using Eq. \eqref{eq:n-1-range} and \eqref{eq:optimum}
			we can write:
			\begin{equation}
			\nonumber
			n^{*}_{h}(\mathcal{T}^{*}_{N}) - n^{*}_{h}(\mathcal{T}^{*}_{N-1}) \le 1
			\end{equation}
			which proves our lemma.
		\end{proof}
	\end{lemma}

	\begin{lemma}
		\label{lem:sub-graph-optimality}
		Let $\mathcal{T}^{*}_{N}$ be an optimum-tree of $N$ vertices and $\mathcal{T}^{*}_{N-1}$ be an optimum-tree of $N-1$ vertices. $\mathcal{T}^{*}_{N-1}$ is an induced subgraph of $\mathcal{T}^{*}_{N}$.
		\begin{proof}
			Let us choose a node to be removed from $\mathcal{T}^{*}_{N}$ and join its neighboring nodes to obtain a tree $\mathcal{T}_{N-1}$ with $N-1$ nodes with an objective of ensuring a maximum $n^{*}_{h}(\mathcal{T}_{N-1})$. To do so, we can only remove a leaf node from $\mathcal{T}^{*}_{N}$. This is because removing non-leaf nodes can reduce the shortest path between multiple pairs of nodes whereas removing leaf-nodes will reduce only the shortest path to nodes from the new leaf at that position. This ensures least reduction in optimum-nodes for $\mathcal{T}_{N-1}$.
			Removing a leaf node implies that $n^{*}_{h}(\mathcal{T}_{N-1})$ cannot be lesser than $n^{*}_{h}(\mathcal{T}^{*}_{N})-1$ as it affects only the paths involving that particular leaf node. Using Lemma \ref{lem:atmost-one}, we see that $\mathcal{T}_{N-1}$ is equivalent to $\mathcal{T}^{*}_{N-1}$, i.e., $\mathcal{T}_{N-1}$ is one of the possible optimal-trees with $N-1$ nodes. Since $\mathcal{T}_{N-1}$ was formed by removing a leaf node from $\mathcal{T}^{*}_{N}$, we find that $\mathcal{T}^{*}_{N-1}$ is indeed an induced subgraph of $\mathcal{T}^{*}_{N}$.

		\end{proof}
	\end{lemma}
	
	\begin{figure}[t]
		\centering
		\subfloat[Starlike tree]{
			\includegraphics[width=0.22\textwidth]{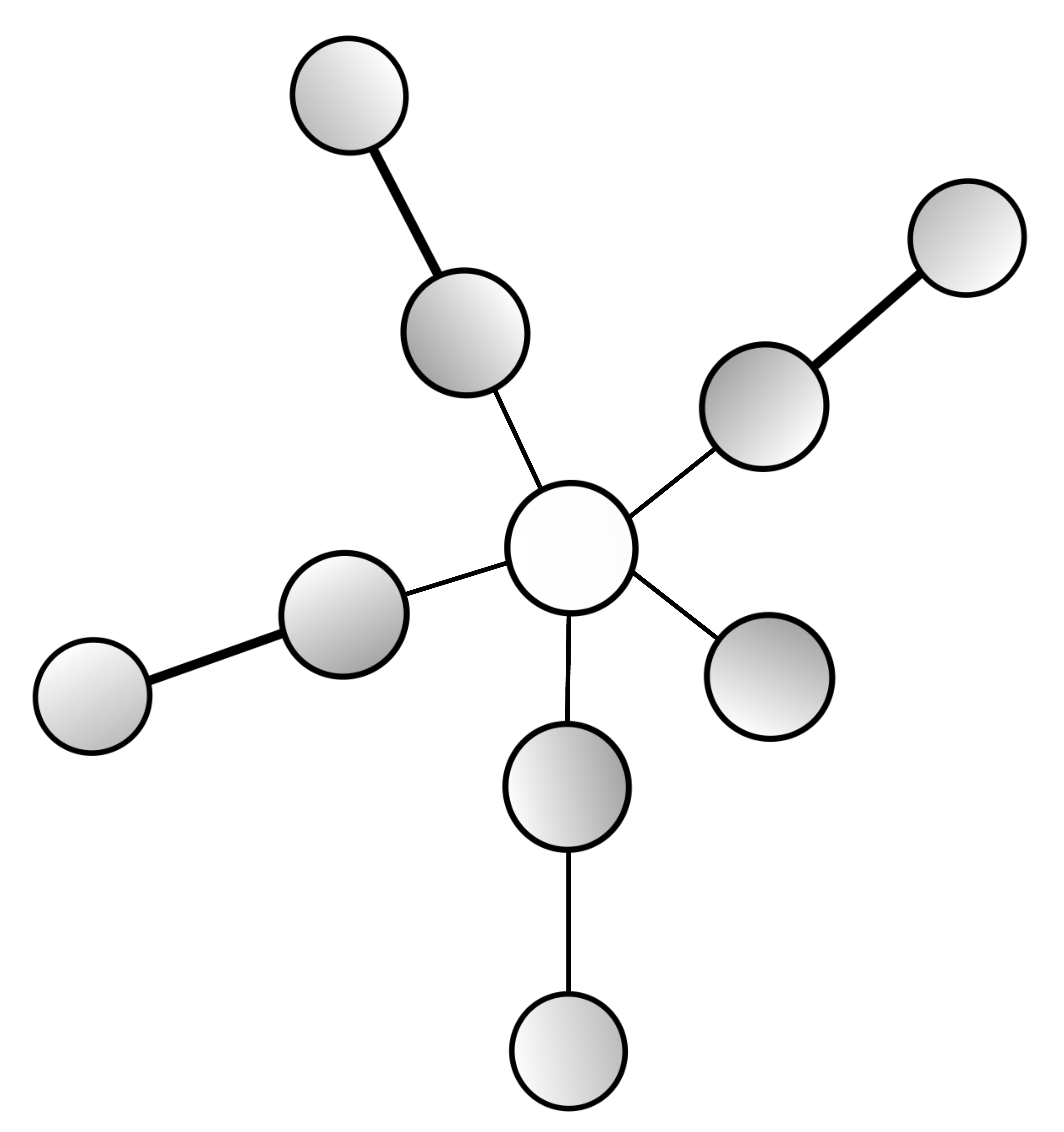}
			\label{starlike}
		}
		\subfloat[Path Graph]{
			\includegraphics[width=0.22\textwidth]{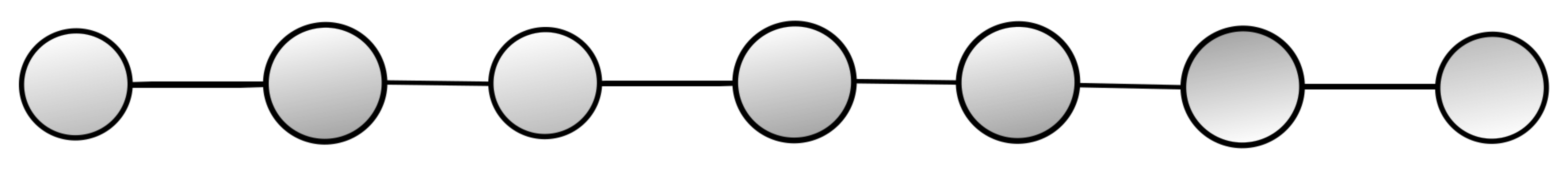}
			\label{path}
		}
		\caption{ (a) Balanced Starlike tree with height 2. (b) Path Graph}
		\label{fig:graphs}
	\end{figure}
	
	
	\begin{definition}
		A starlike tree is a tree having atmost one node (root) with degree greater than two \cite{starlike-tree}. We consider starlike tree with height $\Bigl \lceil h/2 \Bigl \rceil$ to be \textit{balanced}, if there is atmost one leaf which is at a height less than $\Bigl \lceil h/2 \Bigl \rceil$ while the rest are all at a height $\Bigl \lceil h/2 \Bigl \rceil$ from the root. Figure \ref{fig:graphs}(a) depicts an example of a balanced starlike tree with $h=2$.
	\end{definition}
	
	\begin{definition}
		A path graph is a graph such that its nodes can be placed on a straight line. There are \st{not more than} only two nodes in a path graph which have degree one while the rest have a degree of two. Figure \ref{fig:graphs}(b) shows an example of a path graph \cite{path-graph}. 
	\end{definition}

	\begin{lemma}
		\label{lem:starlike-n}
		For a balanced starlike tree with height $h/2$, where $h$ is even, $n^{*}_{h}(\mathcal{T}_{N}) = \Bigl \lfloor \frac{N-1}{\frac{h}{2}} \Bigr \rfloor$, i.e., when the leaves are selected.
	\end{lemma}
	
	\begin{lemma}
		\label{lem:star}
		Among all the possible trees $\mathcal{T}_{N}$ which have N vertices, the maximum $n^{*}_{h}(\mathcal{T}_{N})$ achievable is $ \Bigl \lfloor \frac{N-1}{\frac{h}{2}} \Bigr \rfloor$, which is obtained if the tree is a balanced starlike tree with height $h/2$ if $h$ is even.
		\begin{proof} 
			 To prove the lemma, we use induction. Here, the base case corresponds to a path graph $\mathcal{T}_{h+1}$ with $h+1$ nodes, a trivial case of starlike graph, as it has only $2$ nodes which are $h$ hops away. From the formula $\Bigl \lfloor \frac{N-1}{\frac{h}{2}} \Bigr \rfloor$, we get $n^{*}_{h}(\mathcal{T}_{h+1}) = 2$ which verifies the base case.

			 For any $N-1$, let us assume that the lemma is true, i.e., a balanced starlike tree with height $h/2$ achieves the maximum $n^{*}_{h}(\mathcal{T}_{N-1})$ for any tree $\mathcal{T}_{N-1}$ with $N-1$ vertices. Consider $\mathcal{T}^{*}_{N}$ to be the optimal-tree for $N$ nodes. From Lemma \eqref{lem:sub-graph-optimality}, we know that $\mathcal{T}^{*}_{N-1}$ is an induced subgraph of $\mathcal{T}^{*}_{N}$. This means that $\mathcal{T}^{*}_{N}$ can be obtained by adding a node to $\mathcal{T}^{*}_{N-1}$. Since we are constructing $\mathcal{T}^{*}_{N}$, we need to add a node to $\mathcal{T}^{*}_{N-1}$ such that maximum nodes can be selected which are atleast $h$ hops away. There are three possible structures for the tree $\mathcal{T}^{*}_{N-1}$ depending on the minimum height among all its branches: (a) minimum height among all the branches is less than $h/2-1$, (b) minimum height among all the branches is equal to $h/2-1$ and (c) minimum height among all the branches is equal to $h/2$. Although case (a) is not possible as we assumed $\mathcal{T}^{*}_{N-1}$ to be a balanced starlike tree, we consider it for the sake of completeness. For case (a), no matter where we add the node, $n^{*}_{h}(\mathcal{T}^{*}_{N})$ will not increase. However, we should add the node to the leaf of the branch with least height as it will allow the new leaf of that branch to be chosen in case the number of nodes in tree is increased to some $N^{'}>N$ such that height of that branch becomes $h/2$. For case (b), we should add the node to the leaf of the branch with least height so that its height becomes $h/2$ and the new leaf of that branch gets selected. For case (c), no matter where we add the node, $n^{*}_{h}(\mathcal{T}^{*}_{N})$ will not increase. Unlike case (a), we should add the new node to the root so as to start a new leaf which could be selected if that branch grows to a height $h/2$ for some $N^{'}>N$. For all the three cases, $\mathcal{T}^{*}_{N}$ is a balanced starlike tree as the new node is either added to the leaf of a branch if minimum height of a leaf is less than $h/2$ or to the root if the minimum height of the branches is $h/2$. Hence, by induction, the lemma is proved.
			
		\end{proof}
	\end{lemma}

	\begin{figure}[!ht]
	\centering
	\includegraphics[width=0.4\textwidth]{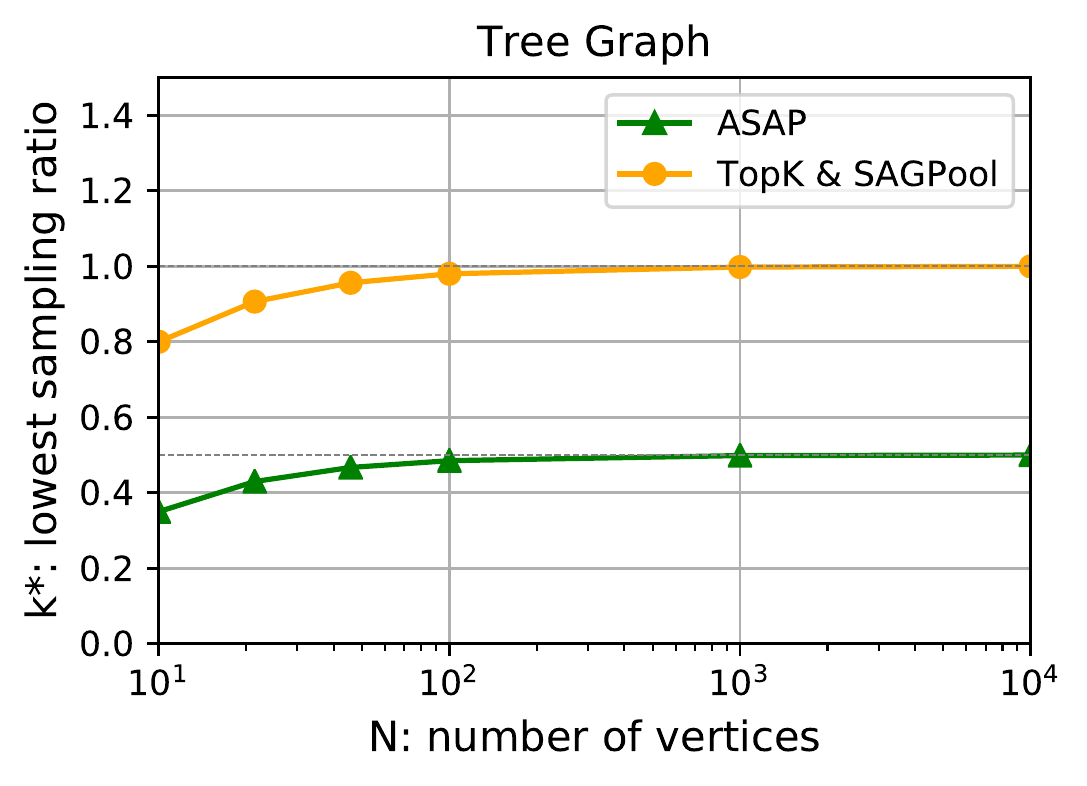}
	\caption{\label{fig:tree-graph} Minimum sampling ratio $k^{*}$ vs $N$ for Path-Graph}
	\end{figure}

	\textbf{Theorem 2.}
	\textit{
		Let the input graph $\mathcal{T}$ be a tree of any possible structure with $N$ nodes. Let $k^{*}$ be the lower bound on sampling ratio $k$ to ensure the existence of atleast one edge in the pooled graph irrespective of the structure of $\mathcal{T}$ and the location of the selected nodes. For TopK or SAGPool, $k^{*} \rightarrow 1$ whereas for ASAP, $k^{*} \rightarrow 0.5$ as $N \rightarrow \infty$.
	}
	\begin{proof}
		From Lemma \eqref{lem:star} and \eqref{lem:starlike-n}, we know that among all the possible trees $\mathcal{T}_{N}$ which have N vertices, the maximum $n^{*}_{h}(\mathcal{T}_{N})$ achievable is $ \Bigl \lfloor \frac{N-1}{\frac{h}{2}} \Bigr \rfloor$. Using pigeon-hole principle we can show that for a pooling method with $RF^{edge}$ if the number of sampled clusters is greater than $n^{*}_{RF^{edge} }(\mathcal{T}_{N})$ then there will always be an edge in the pooled graph irrespective of the position of the selected clusters: 
		\begin{equation}
		\begin{split}
		\label{eq:star-graph}
		\lceil k^{*}N \rceil &> \Bigl \lfloor \frac{N-1}{\frac{ RF^{edge}+1}{2}} \Bigr \rfloor  \\
		k^{*}N &> \Bigl \lfloor \frac{N-1}{\frac{RF^{edge}+1}{2}} \Bigr \rfloor \\
		k^{*}N + 1 &> \frac{N-1}{\frac{RF^{edge}+1}{2}}
		\end{split}
		\end{equation}
		Let us consider 1-hop neighborhood for pooling, i.e., $h = 1$. Substituting $RF^{edge} = h$ in Eq. \eqref{eq:star-graph} for TopK and SAGPool we get:
		\begin{equation}
		\nonumber
		k^{*} > 1 - \frac{2}{N}
		\end{equation}
		and as N $\rightarrow \infty$ we obtain $k^{*} \rightarrow 1$. Substituting $RF^{edge} = 2h+1$ in Eq. \eqref{eq:star-graph} for ASAP we get:
		\begin{equation}
		\nonumber
		k^{*} > \frac{1}{2} - \frac{3}{2N}
		\end{equation}
		and as N $\rightarrow \infty$ we obtain $k^{*} \rightarrow 0.5$
	\end{proof}


	\subsection{\Large Similar Analysis for Path Graph}
	\label{ssec:proof-3}

	\begin{lemma}
		\label{lem:path-graph-n}
		For a path graph $\mathcal{G}_{path}$ with $N$ nodes, $n^{*}_{h}({G}_{path}) = \lceil \frac{N}{h} \rceil$.
	\end{lemma}

	\begin{lemma}
		\label{lem:path-graph}
		Consider the input graph to be a path graph $\mathcal{G}_{path}$ with $N$ nodes. To ensure that a pooling operator with $RF^{edge}$ and sampling ratio $k$ has at least one edge in the pooled graph, irrespective of the location of selected clusters, we have the following inequality on k: $k \geq (\frac{1}{RF^{edge}+1} + \frac{1}{N})$.
	\end{lemma}
	\begin{proof}
		From Lemma \eqref{lem:path-graph-n}, we know that $n^{*}_{RF^{edge}}({G}_{path}) = \lceil \frac{N}{RF^{edge}} \rceil$. Using pigeon-hole principle we can show that for a pooling method with $RF^{edge}$, if the number of sampled clusters is greater than $n^{*}_{RF^{edge}}({G}_{path})$, then there will always be an edge in the pooled graph irrespective of the position of the selected clusters: 
		\begin{equation}
		\label{eq:path-graph-pigeon}
		\begin{split}
		\lceil kN \rceil &> \Bigl \lceil \frac{N}{RF^{edge}+1} \Bigr \rceil \\
		kN &> \Bigl \lceil \frac{N}{RF^{edge}+1} \Bigr \rceil \\
		kN &\geq \frac{N}{RF^{edge}+1} + 1
		\end{split}
		\end{equation}
		From Eq. \eqref{eq:path-graph-pigeon}, we get $k \geq (\frac{1}{RF^{edge}+1} + \frac{1}{N})$ which completes the proof.
	\end{proof}

	\begin{figure}[!ht]
	\centering
	\includegraphics[width=0.4\textwidth]{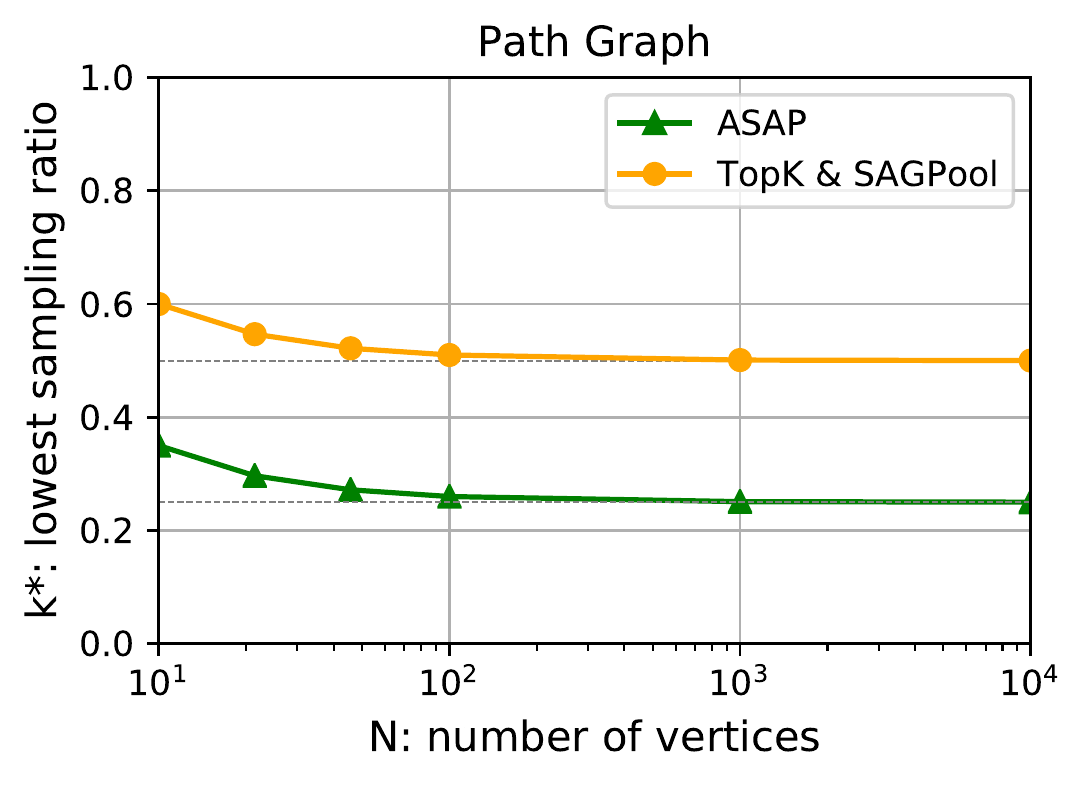}
	\caption{\label{fig:path-graph} Minimum sampling ratio $k^{*}$ vs $N$ for Path-Graph}
	\end{figure}

	\textbf{Theorem 3.}
	\textit{
		Consider the input graph to be a path graph with $N$ nodes. Let $k^{*}$ be the lower bound on sampling ratio $k$ to ensure the existence of atleast one edge in the pooled graph. For TopK or SAGPool,  $k^{*} \rightarrow 0.5$ as $N \rightarrow \infty$ whereas for ASAP, $k^{*} \rightarrow 0.25$ as $N \rightarrow \infty$.
	}
	\begin{proof}
		From Lemma \eqref{lem:path-graph}, we get $k^{*} = \frac{1}{RF^{edge}+1} + \frac{1}{N}$. Using 	h = 1 for TopK and SAGPool when N tends to infinity i.e. $k^{*} = \lim_{N \rightarrow \infty} \frac{1}{2} + \frac{1}{N}$, we get $k^{*} \rightarrow 0.25$. Using h = 3 for ASAP when N tends to infinity i.e. $k^{*} = \lim_{N \rightarrow \infty} \frac{1}{4} + \frac{1}{N}$, we get $k^{*} \rightarrow 0.25$.
	\end{proof}

	\subsection{Graph Connectivity via $k^{th}$ Graph Power.}
	\label{kth-power}
	To minimize the possibility of nodes getting isolated in pooled graph, TopK employs $k^{th}$ graph power i.e. $\hat{A}^{k}$ instead of $\hat{A}$. This helps in increasing the density of the graph before pooling. While using $k^{th}$ graph power, TopK can connect two nodes which are atmost $k$ hops away whereas ASAP in this setting can connect upto $k+2h$ hops in the original graph.  As $h\geq1$, ASAP will always have better connectivity given $k^{th}$ graph power.

	\section{Graph Permutation Equivariance}
	\label{ssec:perm-eq-proof}
	
	Given a permutation matrix $P \in \{0, 1\}^{n \times n}$ and a function $f(X, A)$ depending on graph with node feature matrix $X$ and adjacency matrix $A$, graph permutation is defined as $f(PX, PAP^{T})$, node permutation is defined as $f(PX, A)$ and edge permutation is defined as $f(X, PAP^{T})$.
	Graph pooling operations should produce pooled graphs which are isomorphic after graph permutation i.e., they need to be graph permutation equivariant or invariant. We show that ASAP has the property of being graph permutation equivariant.\\
	
	\textbf{Proposition 1.}
	\textit{ASAP is a graph permutation equivariant pooling operator.}
	\begin{proof}
		Since $S$ is computed by an attention mechanism which attends to all edges in the graph, we have:
		\begin{equation}
		\label{eq:graph-eq-S}
		S \rightarrow PSP^{T}
		\end{equation}
		Selecting top $\lceil kN \rceil$ clusters denoted by indices $i$, changes $\hat{S}$ as:
		\begin{equation}
		\label{eq:graph-eq-shat}
		\hat{S} \rightarrow P\hat{S}(P[i, i])^{T}    
		\end{equation}
		Using Eq. \eqref{eq:graph-eq-shat} and    
		$\hat{S} = S(:, \hat{i}), X^{p} = \hat{X}^{c}(\hat{i})$, we can write:
		\begin{equation}
		\label{eq:graph-Xp}
		X^{p} \rightarrow P[i, i]X^{p}
		\end{equation}
		Since $A^{p} = \hat{S}^{T} \hat{A}^c \hat{S}$ and $\hat{S} = S(:, \hat{i}), X^{p} = \hat{X}^{c}(\hat{i})$, we get:
		\begin{equation}
		\label{eq:graph-Ap}
		A^{p} \rightarrow P[i, i]A^{p}(P[i, i])^{T}
		\end{equation}
		From Eq. \eqref{eq:graph-Xp} and Eq. \eqref{eq:graph-Ap}, we see that graph permutation does not change the output features. It only changes the order of the computed feature and hence is isomorphic to the pooled graph.
	\end{proof}

	\section{Pseudo Code}
	\label{sec:pseudocode}
	Algorithm \ref{alg:ASAP} is a pseudo code of ASAP. The Master2Token working is explained in Algorithm \ref{alg:M2T}.
	
	\begin{algorithm}[!ht]
		\caption{ASAP algorithm}
		\label{alg:ASAP}
		\SetKwInOut{Input}{Input}\SetKwInOut{Output}{Output}\SetKwInOut{Intermediate}{Intermediate}
		\Input{~Graph $G(\V,\E)$; Node features $X$; Weighted adjacency matrix $A$; Master2Token attention function $\textsc{Master2Token}$; Local Extrema Convolution operator $\textsc{LEConv}$; pooling ratio $k$; Top-k selection operator \textsc{TopK}; non-linearity $\sigma$ }
		\Intermediate{~Clustered graph $G^c(\V,\E)$ with node features $X^c$ and weighted adjacency matrix $A^c$; Cluster assignment matrix $S$; Cluster fitness vector $\Phi$}
		\Output{~Pooled graph $G^p$ with node features $X^p$ and weighted adjacency matrix $A^p$}
		\BlankLine
		$X^c, S \leftarrow \textsc{Master2Token}(X, A)$\;
		$A^c \leftarrow A$\;
		$\Phi \leftarrow \textsc{LEConv}(X^c, A^c)$\;
		$\hat{X^c} \leftarrow \Phi \odot X^c $\;
		$\hat{i} \leftarrow \textsc{TopK}(\Phi, k)$\;
		$\hat{S} \leftarrow S(:,\hat{i})$\;
		$X^p \leftarrow \hat{X^c}(\hat{i},:)$\;
		$A^p \leftarrow \hat{S}^{T} \hat{A}^c \hat{S}$
	\end{algorithm}

\makeatletter
\newcommand{\removelatexerror}{\let\@latex@error\@gobble}
\makeatother

\newcommand{\myalgorithm}{%
	\begingroup
	\removelatexerror
	\begin{algorithm*}[H]
		
		\caption{\textsc{Master2Token} algorithm}
		\label{alg:M2T}
		\SetKwInOut{Input}{Input}\SetKwInOut{Output}{Output}
		\Input{~Graph $G(\V,\E)$; Node features $X$; Weighted adjacency matrix $A$; Graph Convolution operator $\textsc{GCN}$; Weight matrix $W$, weight vector $\Vec{w}$; softmax function $softmax$; Cluster neighborhood function $c_{h}$ }
		\Output{~Clustered graph node features $X^c$; Cluster assignment matrix $S$}
		\BlankLine
		$X' \leftarrow \textsc{GCN}(X, A)$\;
		\For{$i=1...|\V|$}
		{
			$x_{i}^{c} \leftarrow \Vec{0}$\;
			$m_i \leftarrow \max_{j \in c_{h}(v_{i})}(x_{j}').$\;
			\For{$j \in c_{h}(v_{i})$}
			{
				$\alpha_{i, j} \leftarrow softmax(\Vec{w}^{T}\sigma(W m_{i} \mathbin\Vert x_{j}'))$\;
				$S_{i, j} \leftarrow \alpha_{i, j}$\;
				$x_{i}^{c} \leftarrow x_{i}^{c} + \alpha_{i,j} x_{j}$\;
			}
		}
	\end{algorithm*}
	\endgroup}

\myalgorithm

\end{document}